\documentclass[twoside,11pt]{article}

\usepackage[utf8]{inputenc}
\usepackage[T1]{fontenc}
\usepackage{booktabs}
\usepackage{mathtools}
\usepackage{amsmath}
\usepackage{lastpage}
\usepackage{multirow}
\usepackage{jmlr2e}

\newcommand{\indep}{\perp \!\!\! \perp}

\newtheorem{assumption}{Assumption}

\jmlrheading{XX}{2026}{1-\pageref{LastPage}}{[Preprint]}{[Unpublished]}{}{Álvaro Troyano et al}

\ShortHeadings{Tensor Network Moral Graph Recovery of Discrete Probability Distributions}{A. Troyano et al.}
\firstpageno{1}

\begin{document}
\title{Tensor Network Moral Graph Recovery of Discrete Probability Distributions}

\author{\name Á. Troyano Olivas \\
        \addr Heisenberg Research Center (Munich), Huawei Technologies Duesseldorf GmbH, Germany \\
        Center for Computational Simulation, Universidad Politécnica de Madrid, Madrid, Spain \\
        \email alvaro.troyano@huawei.com
        \AND
        \name Chi-Hang Fred Fung \thanks{Corresponding author.} \\
        \addr Heisenberg Research Center (Munich), Huawei Technologies Duesseldorf GmbH, Germany \\
        \email fred.fung@huawei.com
        \AND
        \name Hans H. Brunner \\
        \addr Heisenberg Research Center (Munich), Huawei Technologies Duesseldorf GmbH, Germany \\
        \email hans.brunner@huawei.com
        \AND
        \name Momtchil Peev \\
        \addr Heisenberg Research Center (Munich), Huawei Technologies Duesseldorf GmbH, Germany \\
        \email momtchil.peev@huawei.com
        \AND
        \name Vicente Martin \\
        \addr Center for Computational Simulation, Universidad Politécnica de Madrid, Madrid, Spain \\
        \email vicente.martin@upm.es
        }
        
\editor{[--]}

\maketitle
\begin{abstract}
We present a method for recovering the moral graph of a causal DAG from a probability distribution over discrete variables, using fully connected tensor networks (FCTNs) with nuclear-norm-regularized bond corrections. Each bond matrix is parameterized as a baseline all-ones matrix plus a low-rank correction $C_{ij} = U_{ij}V_{ij}^\top$, and the nuclear norm of the correction implemented via the variational Frobenius norm penalty on the factors drives unnecessary bonds to zero. We prove that under faithfulness, positivity, and a no-implicit-rerouting assumption on the local tensor architecture, \textbf{every} optimal FCTN with zero reconstruction error $\varepsilon = 0$ has effective graph exactly equal to the moral graph. For the approximate regime ($\varepsilon > 0$), we provide explicit recovery bounds using the Fannes--Audenaert continuity of conditional mutual information, and derive a sufficient condition on the regularization parameter $\beta$. The effective graph is read directly from the optimized bond matrices.
\end{abstract}

\section{Introduction}
\label{sec:introduction}

Recovering the causal structure of a system of discrete random variables from observational data is a central problem in causal discovery. The moral graph, an undirected graph that connects two variables whenever they share a direct causal link or are co-parents in a v-structure, is a fundamental intermediate representation: it is an invariant of the Markov equivalence class (see Appendix~\ref{sec:causal_background} for an introduction to the moral graph) and can be determined from purely observational data, without interventions.

Existing approaches to structure learning fall broadly into two categories. \emph{Constraint-based} methods such as the PC algorithm
\cite{spirtes_2000, meek_rules} iteratively remove edges from a complete graph using conditional independence (CI) tests; their accuracy depends critically on the reliability of these tests, which degrade in the presence of finite samples or weak dependencies. \emph{Score-based} methods such as GES and NOTEARS \cite{zheng_2018} reformulate structure learning as continuous optimization over a score function, but search over the space of DAGs, which is combinatorially large, and typically require acyclicity constraints or surrogate penalties.

In this work we propose a fundamentally different approach: we decompose the observed probability distribution $p(V_1, \ldots, V_m)$ into a \emph{fully connected tensor network} (FCTN) whose bond matrices are parameterized as a baseline all-ones matrix plus a low-rank correction $C_{ij} = U_{ij}V_{ij}^\top$. A nuclear-norm penalty on the corrections (implemented via the variational Frobenius norm on the factors) drives unnecessary bonds to zero, so that the \emph{effective graph} (the set of bonds with nonzero correction) emerges directly from optimization.

The key ideas are:
\begin{itemize}
    \item The nuclear-norm penalty on bond corrections replaces discrete CI tests with a continuous, differentiable measure of bond strength. Rerouting a correlation through an intermediate node is always more expensive than representing it directly, so the optimizer prefers direct bonds for genuine conditional dependencies (Lemma~\ref{lem:rerouting_cost}).
    
    \item Non-moral edges carry zero conditional mutual information. Under the no-implicit-rerouting assumption, any correction on such an edge is either null, implicitly rerouted (ruled out by the architectural constraint), or explicitly rerouted (which is strictly suboptimal by the rerouting cost). Consequently, non-moral edges are driven to zero correction (Theorem~\ref{th:upper_bound}). 
    
    \item For approximate reconstruction ($\varepsilon > 0$), the Fannes--Audenaert continuity bound on conditional mutual information \cite{audenaert_2007} provides explicit recovery guarantees and a sufficient condition on the regularization parameter $\beta$     (Theorem~\ref{th:beta}).
\end{itemize}

Our main result (Theorem~\ref{th:exact_recovery}) states that under faithfulness, positivity, and a no-implicit-rerouting assumption on the local tensor architecture, \textbf{every} optimal FCTN with zero reconstruction error has effective graph exactly equal to the moral graph. This is a uniqueness statement: the result holds for all optimal solutions, not merely for some existence result. The approximate regime is handled by Theorem~\ref{th:approx_recovery}, which bounds the conditional mutual information of spurious edges and shows convergence to the moral graph as $\varepsilon \to 0$.

The remainder of the paper is organized as follows. Section~\ref{sec:problem} states the problem formally. Section~\ref{sec:definitions} introduces the necessary definitions (distributions, information quantities, bond matrices, effective graph, d-separation). Section~\ref{sec:objective} defines the objective function and the variational nuclear-norm penalty. Section~\ref{sec:theory} establishes the information-theoretic and graphical lemmas underlying the recovery proofs. Section~\ref{sec:tn_properties} develops the tensor-network properties (rewiring, rerouting cost, implicit rerouting). Sections~\ref{sec:exact_recovery} and \ref{sec:approx_recovery} state and prove the exact and approximate recovery theorems, respectively. Section~\ref{sec:experiments_synthetic} reports synthetic experiments. Section~\ref{sec:discussion} discusses limitations and connections to existing work, and Section~\ref{sec:conclusions} concludes. Appendix~\ref{sec:causal_background} reviews the necessary background on structural causal models and appendix~\ref{sec:choosing_beta} gives a principled way to choose the regularization parameter $\beta$.

\section{Problem Statement}
\label{sec:problem}

Let $V = \{V_1, \ldots, V_m\}$ be a set of discrete random variables with joint distribution $p(V_1, \ldots, V_m)$. We assume that $p$ is faithful to an unknown causal DAG $G$ (Assumption~\ref{assum:faithfulness}) with moral graph $G^m = (V, E^m)$. The goal is to recover the edge set $E^m$ from $p$ alone.

Our method decomposes $p$ into a fully connected tensor network (FCTN) whose bond matrices are parameterized as $B_{ij} = J_{ij} + U_{ij}V_{ij}^\top$. A nuclear-norm penalty on the corrections $C_{ij} = U_{ij}V_{ij}^\top$ drives unnecessary bonds to zero, so that the effective graph
\[
    G_{\mathrm{eff}}(TN) = \bigl(V,\;\bigl\{\{V_i,V_j\} : \|C_{ij}\|_* > 0\bigr\}\bigr)
\]
is read directly from the optimized solution. The formal ansatz, objective function, and recovery guarantees are developed in Sections~\ref{sec:definitions}--\ref{sec:approx_recovery}. 

\section{Definitions}
\label{sec:definitions}

\subsection{Distributions, Marginals, and Entropy}
\label{sec:definitions_entropy}
Let $V = \{V_1, \dots, V_m\}$ be a set of $m$ discrete random variables with finite state spaces $\mathcal{X}_{V_i} = \{1, \dots, d_i\}$. For a non-empty subset $S \subseteq V$ we define the \emph{configuration space} of the variables in $S$ by
\begin{equation}
    \mathcal{X}_S := \bigtimes_{i \,:\, V_i \in S} \mathcal{X}_{V_i},
\end{equation}
and we write $p_S$ for the marginal distribution of $p$ \emph{on} $S$:
\begin{equation}
\label{eq:marginal_projection}
p_S(s) \;=\; \sum_{t \in \mathcal{X}_{V \setminus S}} p(s, t),
\qquad s \in \mathcal{X}_S,
\end{equation}
with the convention $p_V \equiv p$ when $S = V$. Thus $p_S$ is the projection of $p$ onto the coordinates of $S$, obtained by summing over all configurations of the complementary variables.

The \emph{Shannon entropy} of a distribution $q$ on a finite set $\mathcal{X}$ is
\[
H(q) \;:=\; -\sum_{x \in \mathcal{X}} q(x)\log_2 q(x),
\]
with the usual convention $0 \log_2 0 := 0$.

To keep the notation compact, we write $H(S) := H(p_S)$ for the entropy of a marginal and use the shorthand $H(A,B) := H(A \cup B)$ for entropies of joints. \emph{All logarithms in this work are base $2$}, i.e., all entropies and information quantities are expressed in bits.

\subsection{Mutual Information and Conditional Mutual Information}
\label{sec:MI_CMI}
For two disjoint non-empty subsets $A, B \subseteq V$, the (\emph{unconditional}) \emph{mutual information} between $A$ and $B$ is defined as
\begin{equation}
\label{eq:MI}
    I(A : B) \;:=\; H(A) + H(B) - H(A,B).
\end{equation}
$I(A,B)$ is non-negative and symmetric in $A$ and $B$, and vanishes if and only if $A$ and $B$ are independent in $p$.

For three pairwise-disjoint subsets $A, B, C \subseteq V$, the \emph{conditional mutual information} of $A$ and $B$ given $C$ is
\begin{equation}
\label{eq:CMI}
    I(A : B \mid C) \;:=\; H(A,C) + H(B,C) - H(C) - H(A,B,C).
\end{equation}
Equivalently, it is the expected mutual information of $A$ and $B$ inside the conditional distribution $p_{A,B \mid C}$,
\[
I(A : B \mid C) \;=\; \mathbb{E}_{\,p_C}\Bigl[\, I\bigl( p_{A \mid C} , \, p_{B \mid C} \bigr)\Bigr],
\]
from which the established \emph{if and only if} characterization follows directly:
\[
I(A : B \mid C) = 0 \quad\iff\quad A \indep B \mid C.
\]
Throughout this work, the central special case is the dependence between a pair of variables when conditioned on \emph{all remaining variables}:
\begin{equation}
\label{eq:pair_cmi}
I(V_i : V_j \mid V_{\mathrm{rest}}) , \qquad V_{\mathrm{rest}} := V \setminus \{V_i, V_j\}.
\end{equation}
Whenever the conditioning set is $V_{\mathrm{rest}}$, the subscript $\mathrm{rest}$ is omitted if no ambiguity can arise.

\subsection{Graphical Structure and d-Separation}
\label{sec:dsep}

Given a DAG $G = (V, E)$, a \emph{trail} between two nodes is a sequence of adjacent edges connecting them, traversed in any direction. A node $W$ on a trail is a \emph{collider} (or \emph{v-structure}) if both edges incident to $W$ on the trail point toward $W$ (i.e.\ $\cdots \to W \leftarrow \cdots$); otherwise $W$ is a \emph{non-collider} on that trail.

\begin{definition}[d-Separation]
\label{def:dsep}
Let $G = (V, E)$ be a DAG and let $X, Y, Z \subseteq V$ be pairwise disjoint. A trail between a node in $X$ and a node in $Y$ is \emph{blocked} by $Z$ if it contains either:
\begin{enumerate}
    \item a non-collider $W$ with $W \in Z$, or
    \item a collider $W$ such that neither $W$ nor any descendant of $W$
          in $G$ belongs to $Z$.
\end{enumerate}
$X$ and $Y$ are \emph{d-separated} given $Z$, written $\mathrm{dsep}_G(X;\, Y \mid Z)$, if every trail between $X$ and $Y$ is blocked by $Z$. If not every trail is blocked, $X$ and $Y$ are \emph{d-connected} given $Z$, denoted $\neg\,\mathrm{dsep}_G(X;\, Y \mid Z)$.
\end{definition}

The d-separation criterion encodes the conditional independence structure implied by a DAG: for any distribution $p$ Markovian to $G$,
$\mathrm{dsep}_G(X;\, Y \mid Z) \implies X \indep Y \mid Z$. Under faithfulness (Definition in Section~\ref{sec:theory}), the converse also holds, establishing the bidirectional correspondence $\mathrm{dsep}_G(X;\, Y \mid Z) \iff X \indep Y \mid Z$ that underpins the lemmas of Section~\ref{sec:theory}.

\subsection{Bond Matrices and the Tensor Network Ansatz}
\label{sec:bond_matrices}
A \textit{fully connected tensor network} (FCTN) over $m$ sites represents a joint distribution $p(s_1, \ldots, s_m)$ as a contraction of $m$ local tensors $N^{[i]}$ connected by bond matrices $B_{ij}$ on every pair $\{i,j\}$:
\begin{equation}
    \label{eq:ansatz}
    TN(s_1, \ldots, s_m) = \sum_{\{\alpha_{ij}\}} \prod_{i=1}^m N^{[i]}_{s_i, \{\alpha_{ij}: j \neq i\}} \prod_{\{i,j\}} B_{ij}(\alpha_{ij}, \alpha_{ji}),
\end{equation}
where $s_i \in \{1, \ldots, d_i\}$ is the physical index at site $i$, $\alpha_{ij} \in \{1, \ldots, r_{\max}\}$ is the bond index connecting sites $i$ and $j$, and $r_{\max}$ is the maximum bond dimension (fixed and uniform across all bonds).

\subsubsection{Baseline-plus-correction parameterization}

Each bond matrix is decomposed as:
\begin{equation}
\label{eq:bond_param}
B_{ij} = J_{ij} + C_{ij}, \qquad C_{ij} = U_{ij} V_{ij}^\top, \qquad U_{ij} \in \mathbb{R}^{r_{\max} \times K}, \quad V_{ij} \in \mathbb{R}^{r_{\max} \times K},
\end{equation}
where $J_{ij}$ is the $r_{\max} \times r_{\max}$ all-ones matrix (a fixed constant, not a trainable parameter) and $C_{ij} = U_{ij}V_{ij}^\top$ is the \textbf{correction} that deviates the bond from the disconnected baseline. When $C_{ij} = 0$, the bond reduces to $B_{ij} = J_{ij}$, and the two sites sum their bond indices independently:
\begin{equation}
    \sum_{\alpha_{ij}, \alpha'_{ji}} N^{[i]}(s_i, \alpha_{ij})\, J(\alpha_{ij}, \alpha'_{ji})\, N^{[j]}(s_j, \alpha'_{ji}) = \Bigl(\sum_{\alpha_{ij}} N^{[i]}(s_i, \alpha_{ij})\Bigr)\Bigl(\sum_{\alpha'_{ji}} N^{[j]}(s_j, \alpha'_{ji})\Bigr),
\end{equation}
which is a product of marginals. The bond is \emph{absent}: no information is shared between the two sites. This is the correct notion of ``disconnected'' for a multiplicative contraction unlike $B_{ij} = 0$, which would zero out the entire TN output.
The $J$ baseline ensures the TN output is nonzero even when all corrections are zero (it produces the fully factorized distribution $\hat{p} = \prod_i f_i(s_i)$), so the optimizer cannot trivially zero out the output by collapsing all bonds.

\subsubsection{Role of each object} 

The local tensors $N^{[i]}$ carry the physical indices $s_i$ and connect the tensor network to the variables being modeled. The bond factors $U_{ij}, V_{ij}$ live entirely on the bonds and have no physical indices; they control \emph{whether and how strongly} two sites deviate from independence. The bond matrices can always be absorbed into the local tensors (by setting $\tilde{N}^{[i]} = N^{[i]} \cdot B_{ij}$), so the ansatz is equally expressive as a standard TN. The purpose of keeping them separate is to expose the bond strength for penalization (see Section~\ref{sec:objective}).

The \textit{effective rank} of bond $\{V_i, V_j\}$ is $\mathrm{rank}(C_{ij})$, which the optimizer determines by driving unnecessary singular values of $C_{ij}$ to zero. The \textit{nuclear norm} $\|C_{ij}\|_* = \sum_k \sigma_k(C_{ij})$ measures the total information capacity of the bond (the magnitude of the deviation from the disconnected state).
\subsection{Effective Graph}
\label{sec:effective_graph}
\begin{definition}[Effective Graph]
    Let $TN$ be a tensor network with bond matrices $\{B_{ij} = J_{ij} + C_{ij}\}$. The \textbf{effective graph} of $TN$ is the undirected graph consisting of all bonds with nonzero correction nuclear norm:
    \begin{equation}
        G_{\text{eff}}(TN) = (V, E_{\text{eff}}), \qquad E_{\text{eff}} = \bigl\{\{V_i, V_j\} : \|C_{ij}\|_* > 0\bigr\}.
    \end{equation}
    Bonds with $\|C_{ij}\|_* = 0$ (i.e., $B_{ij} = J_{ij}$, the disconnected baseline) carry no shared information and do not appear in the effective graph.
\end{definition}
The effective graph is read directly from the optimized bond matrices: it is the set of bonds that survived the optimization with non-trivial correction nuclear norm. The nuclear norm penalty in the objective function (Section~\ref{sec:objective}) drives unnecessary bonds to $C_{ij} = 0$ (i.e., $B_{ij} = J_{ij}$), so the effective graph is the direct output of the optimizer.
Compared to a bond-dimension threshold ($r_{ij} > 1$), the nuclear norm of the correction provides a \emph{continuous} measure of bond strength: a bond with $r = 5$ but tiny correction singular values (total nuclear norm $\approx 0$) is correctly pruned, whereas a dimension-based criterion would keep it. In the approximate regime ($\varepsilon > 0$), the threshold becomes $\|C_{ij}\|_* > \tau$ with $\tau = f(\varepsilon^*)$ (Section~\ref{sec:approx_recovery}).

\subsubsection{Practical implementation: rank-based pruning}

In practice, the effective graph is computed using the \emph{effective rank} of the correction (the number of singular values of $C_{ij} = U_{ij}V_{ij}^\top$ exceeding a numerical tolerance $\text{tol} = 10^{-6}$):
\[
r_{\text{eff}}(C_{ij}) = \bigl|\{k : \sigma_k(C_{ij}) > \text{tol}\}\bigr|.
\]
A bond is included in $G_{\text{eff}}$ iff $r_{\text{eff}}(C_{ij}) \geq 1$. This rank-based criterion is preferred over a nuclear norm threshold for two reasons: (1)~the tolerance $10^{-6}$ is stable across problems because it exploits the natural gap between floating-point/optimizer noise (singular values $\sim 10^{-8}$) and genuine corrections (singular values $\sim 10^{-3}$ or larger); (2)~the nuclear norm threshold would need problem-dependent tuning, since the nuclear norm of a moral edge scales with the strength of the dependencies, the value of $\beta$, and the number of inner steps. The nuclear norm is instead used as a \emph{diagnostic} (reporting bond strength) rather than a decision criterion.
\section{Objective Function}
\label{sec:objective}
\begin{definition}[Objective Function]
    \label{def:objective}
    Let $p$ be a probability distribution over $m$ discrete variables. Define the \textbf{objective function}:
    \begin{equation}
        J(\theta) = \bigl\|p - TN(\theta)\bigr\|_2^2 + \frac{\beta}{2} \sum_{\{V_i, V_j\}} \Bigl(\|U_{ij}\|_F^2 + \|V_{ij}\|_F^2\Bigr),
    \end{equation}
    where $\beta > 0$ is a regularization parameter, $TN(\theta)$ is computed using $B_{ij} = J_{ij} + U_{ij}V_{ij}^\top$ on each bond, and $\theta = \{N^{[i]}, U_{ij}, V_{ij}\}$ collects all trainable TN parameters. The $J_{ij}$ baselines are fixed constants. The optimal inner value is $J^* = \min_\theta J(\theta)$.
\end{definition}

Throughout, we write
\begin{equation}
    \varepsilon(\theta) := \bigl\|p - TN(\theta)\bigr\|_2^2,
    \qquad
    P(\theta) := \frac{1}{2} \sum_{\{V_i, V_j\}} \Bigl(\|U_{ij}\|_F^2 + \|V_{ij}\|_F^2\Bigr),
    \label{eq:eps_penalty_split}
\end{equation}
so that $J(\theta) = \varepsilon(\theta) + \beta\, P(\theta)$; at a balanced
factorization (Eq.~\eqref{eq:variational}), $P(\theta) = \sum_{\{V_i,V_j\}} \|C_{ij}\|_*$.

The penalty on $U_{ij}, V_{ij}$ implements the nuclear norm of the correction $C_{ij} = U_{ij}V_{ij}^\top$ via the variational identity:
\begin{equation}
    \label{eq:variational}
    \|C_{ij}\|_* = \min_{C_{ij} = UV^\top} \frac{1}{2}\bigl(\|U\|_F^2 + \|V\|_F^2\bigr).
\end{equation}
This identity states that the nuclear norm of the correction (sum of its singular values) equals the minimum Frobenius cost over all factorizations. The minimum is achieved at the balanced factorization $U = W\Sigma^{1/2}$, $V = Z\Sigma^{1/2}$ (where $C_{ij} = W\Sigma Z^\top$ is the SVD), but the optimizer \emph{discovers} this automatically: any unbalanced factorization has strictly higher Frobenius cost for the same $C_{ij}$, so gradient descent converges to the balanced one. Crucially, no SVD is ever computed during optimization, the $U_{ij}, V_{ij}$ are free parameters initialized randomly and updated by standard backpropagation.
The penalty $\sum \|C_{ij}\|_*$ is a genuine norm on the correction, making the objective structurally identical to LASSO ($\|Ax - b\|^2 + \lambda\|x\|_1$) and matrix completion ($\|P_\Omega(M-X)\|^2 + \lambda\|X\|_*$). The nuclear norm is the tightest convex relaxation of rank, just as $\ell_1$ is the convex relaxation of $\ell_0$. The $J$ baseline plays no role in the penalty: only the deviation $C_{ij}$ is penalized, so the optimizer is incentivized to minimize the correction magnitude while maintaining reconstruction fidelity.
The optimization is performed over the class of \textit{fully connected tensor networks} (FCTNs): the ansatz has all $\binom{m}{2}$ edges present with bond matrices $B_{ij} = J_{ij} + C_{ij}$, and the optimizer drives unnecessary corrections to $C_{ij} = 0$ (i.e., $B_{ij} = J_{ij}$) via the nuclear norm penalty. The effective graph $G_{\text{eff}}$ of the optimal FCTN is the output of the method.

\section{Theoretical Framework}
\label{sec:theory}

To establish the correspondence between moral graph structure and conditional mutual information, we require two standard assumptions.

\begin{definition}[Faithfulness]
    A distribution $P$ is \textbf{faithful} to a causal DAG $G$ if for all sets of variables $X, Y, Z \subseteq V$:
    \[
        (X \indep Y \mid Z) \text{ in } P
        \iff
        \mathrm{dsep}_G(X; Y \mid Z).
    \]
\end{definition}

\begin{assumption}[Faithful Distribution]
    \label{assum:faithfulness}
    We assume that the probability distribution  $p(V_1, \ldots, V_m)$ is faithful to the underlying causal DAG $G$.
\end{assumption}

\begin{assumption}[Positive Distribution]
    \label{assum:positivity}
    We assume $p(s) > 0$ for all configurations $s$. This ensures the applicability of the Hammersley--Clifford theorem \cite{Clifford90} and guarantees that the conditional independence structure is fully captured by the graph.
\end{assumption}

These assumptions are reasonable for several reasons:
\begin{itemize}
    \item \textbf{Generic position}: Almost all parameterizations of a causal DAG are faithful. The set of unfaithful distributions has measure zero in the parameter space \cite{koller_friedman_2009}. Unfaithfulness requires either deterministic relationships, carefully balanced parameter cancellations, or zero probabilities.
    \item \textbf{Computational necessity}: Faithfulness provides the clean correspondence $\emph{d-connected} \iff I > 0$ that our optimization-based structure learning algorithm relies on. Positivity ensures the Hammersley--Clifford factorization over the moral graph.
\end{itemize}

\subsection{Basic Lemmas}

Here we prove some basic lemmas that will be useful later, when proving the main result.

\begin{lemma}[Moral Edges Have Positive Conditional MI]
    \label{lem:moral_mi}
    Let $G$ be a causal DAG with moral graph $G^m = (V, E^m)$. Under Assumption~\ref{assum:faithfulness}, for every edge $\{V_i, V_j\} \in E^m$:
    \[
    I(V_i : V_j \mid V \setminus \{V_i, V_j\}) > 0.
    \]
\end{lemma}
\begin{proof}
    By definition of the moral graph, $\{V_i, V_j\} \in E^m$ if and only if one of the following holds:
    \textbf{Case 1: Direct causal link.} $V_i \rightarrow V_j$ or $V_j \rightarrow V_i$ in $G$.
    When there is a direct edge between $V_i$ and $V_j$, consider the d-separation relationship with conditioning set $Z = V \setminus \{V_i, V_j\}$. The trail $V_i - V_j$ is a direct connection with no intermediate nodes. Consequently, this trail is \textbf{not blocked} by $Z$, meaning $V_i$ and $V_j$ are d-connected given $Z$:
    $$ \neg\, d\text{-sep}_G(V_i; V_j \mid Z) $$
    By the faithfulness assumption (Assumption~\ref{assum:faithfulness}):
    $$ \neg\, d\text{-sep}_G(V_i; V_j \mid Z) \implies \neg\, (V_i \indep V_j \mid Z) $$
    Which is equivalent to:
    $$ I(V_i : V_j \mid V \setminus \{V_i, V_j\}) > 0 $$
    \textbf{Case 2: V-structure.} $V_i \rightarrow V_k \leftarrow V_j$ in $G$, where $V_k$ is a common child and $V_k \notin \{V_i, V_j\}$.
    Consider the trail $V_i - V_k - V_j$. This is a v-structure with collider $V_k$. When conditioning on $Z = V \setminus \{V_i, V_j\}$, the node $V_k$ is in the conditioning set (since $V_k \neq V_i$ and $V_k \neq V_j$).
    A collider trail is \textbf{active} precisely when the collider (or any of its descendants) is in the conditioning set. Therefore, $V_i$ and $V_j$ are d-connected given $Z$:
    $$ \neg\, d\text{-sep}_G(V_i; V_j \mid Z) $$
    By faithfulness:
    $$ \neg\, (V_i \indep V_j \mid Z) \implies I(V_i : V_j \mid V \setminus \{V_i, V_j\}) > 0 $$
    In both cases, the existence of a moral edge implies positive conditional mutual information.
\end{proof}
\begin{lemma}[Non-Moral Edges Have Zero Conditional MI]
    \label{lem:non_moral_mi}
    For every pair $\{V_i, V_j\} \notin E^m$:
    \[
    I(V_i : V_j \mid V \setminus \{V_i, V_j\}) = 0.
    \]
\end{lemma}
\begin{proof}
    If $\{V_i, V_j\} \notin E^m$, then by definition of the moral graph:
    \begin{enumerate}
        \item There is no direct edge $V_i \rightarrow V_j$ or $V_j \rightarrow V_i$ in $G$.
        \item There is no node $V_k$ such that $V_i \rightarrow V_k \leftarrow V_j$ (i.e., $V_i$ and $V_j$ are not co-parents in any v-structure).
    \end{enumerate}
    Let $Z = V \setminus \{V_i, V_j\}$ be the conditioning set. We show that $V_i$ and $V_j$ are d-separated given $Z$.
    Consider any trail in $G$ from $V_i$ to $V_j$. Since there is no direct edge (by (1)), the trail has length at least 2, with intermediate nodes $W_1, W_2, \dots, W_k$.
    For the trail to be \textbf{active} given $Z$:
    \begin{itemize}
        \item Every non-collider $W$ on the trail must satisfy $W \notin Z$
        \item Every collider $W$ on the trail must satisfy $W \in Z$ (or have a descendant in $Z$)
    \end{itemize}
    Since $Z = V \setminus \{V_i, V_j\}$ contains all nodes except $V_i$ and $V_j$, every intermediate node $W$ is in $Z$. For the trail to be active, every non-collider must \textbf{not} be in $Z$---but every intermediate node \textbf{is} in $Z$. Therefore every intermediate node must be a collider.
    But a trail where every intermediate node is a collider requires, for each consecutive pair of edges, both edges to point toward the intermediate node. For a trail $V_i - W_1 - W_2 - \cdots - W_k - V_j$, this means $W_1$ is a collider ($V_i \rightarrow W_1 \leftarrow W_2$ or similar), $W_2$ is a collider, etc. At the junction between $W_1$ and $W_2$, we need both $W_1$ to be a collider (requiring $W_2 \rightarrow W_1$) and $W_2$ to be a collider (requiring $W_1 \rightarrow W_2$), which is impossible in a DAG (it would create a 2-cycle).
    Hence, all trails from $V_i$ to $V_j$ are blocked by $Z$, and:
    $$
        d\text{-sep}_G(V_i; V_j \mid Z) \implies (V_i \indep V_j \mid Z) \implies I(V_i : V_j \mid V \setminus \{V_i, V_j\}) = 0.
    $$
\end{proof}
\begin{lemma}[Conditional Independence Depends Only on Moral Neighbors]
\label{lem:ci_moral_neighbors}
Let $G$ be a causal DAG with moral graph $G^m$, and let $p$ be faithful to $G$ (Assumption~\ref{assum:faithfulness}). For any variable $V_i$ and conditioning set $V_{\text{rest}} = V \setminus \{V_i\}$, the conditional distribution $p(V_i \mid V_{\text{rest}})$ depends only on the moral neighbors of $V_i$ in $V_{\text{rest}}$:
\[
p(V_i \mid V_{\text{rest}}) = p\bigl(V_i \mid \{V_k \in V_{\text{rest}} : \{V_i, V_k\} \in E^m\}\bigr).
\]
\end{lemma}
\begin{proof}
Let $M_i = \{V_k \in V_{\text{rest}} : \{V_i, V_k\} \in E^m\}$ be the moral neighbors of $V_i$ in $V_{\text{rest}}$, and let $N_i = V_{\text{rest}} \setminus M_i$ be the non-moral neighbors. For any $V_k \in N_i$, we have $\{V_i, V_k\} \notin E^m$, so by Lemma~\ref{lem:non_moral_mi}, $I(V_i : V_k \mid V \setminus \{V_i, V_k\}) = 0$, i.e.,
\[
V_i \indep V_k \mid V \setminus \{V_i, V_k\}.
\]
Since $V \setminus \{V_i, V_k\} = M_i \cup (N_i \setminus \{V_k\})$, this becomes
\[
V_i \indep V_k \mid M_i \cup (N_i \setminus \{V_k\}). \tag{$\star$}
\]

We now show $V_i \indep N_i \mid M_i$ by induction on $|N_i|$.
Enumerate $N_i = \{V_{k_1}, \dots, V_{k_n}\}$. For each $\ell = 1, \dots, n$,
equation~($\star$) gives
\[
V_i \indep V_{k_\ell} \mid M_i \cup (N_i \setminus \{V_{k_\ell}\}). \tag{$\star_\ell$}
\]

\emph{Base case} ($|N_i| \leq 1$). If $N_i = \varnothing$ the result is trivial.
If $N_i = \{V_{k_1}\}$, then ($\star_1$) reads $V_i \indep V_{k_1} \mid M_i$,
which is the desired conclusion.

\emph{Inductive step.} Assume for some $j \in \{1, \dots, n-1\}$ that
\[
V_i \indep \{V_{k_1}, \dots, V_{k_j}\} \mid M_i \cup \{V_{k_{j+1}}, \dots, V_{k_n}\}. \tag{IH}
\]
Apply the intersection property of conditional independence (a graphoid axiom valid for strictly positive distributions \cite{koller_friedman_2009, pearl_2009}) with
\[
Y = \{V_{k_1}, \dots, V_{k_j}\}, \quad W = V_{k_{j+1}}, \quad Z = M_i \cup \{V_{k_{j+2}}, \dots, V_{k_n}\}.
\]
The two premises are:
\begin{enumerate}
    \item $V_i \indep Y \mid Z \cup W$, which is~(IH).
    \item $V_i \indep W \mid Z \cup Y$, which is~($\star_{j+1}$).
\end{enumerate}
The conclusion is $V_i \indep (Y \cup W) \mid Z$, i.e.,
\[
V_i \indep \{V_{k_1}, \dots, V_{k_{j+1}}\} \mid M_i \cup \{V_{k_{j+2}}, \dots, V_{k_n}\}.
\]

After $n - 1$ inductive steps, the conditioning set is reduced to $M_i$:
\[
V_i \indep N_i \mid M_i.
\]
Therefore $p(V_i \mid V_{\text{rest}}) = p(V_i \mid M_i \cup N_i) = p(V_i \mid M_i)$.
\end{proof}
\subsection{Continuity of Conditional Mutual Information}
\label{sec:continuity}
A key technical tool is the continuity of conditional mutual information under small perturbations of the distribution.
\begin{lemma}[Continuity of Conditional MI]
    \label{lem:continuity}
    Let $p$ and $\tilde{p}$ be probability distributions over $V$ with $\|p - \tilde{p}\|_1 \leq \varepsilon$. For any partition $\{V_i, V_j\} \mid V_{\text{rest}}$,
    \[
    |I^p(V_i : V_j \mid V_{\text{rest}}) - I^{\tilde{p}}(V_i : V_j \mid V_{\text{rest}})| \leq 4\bigl[\varepsilon \log(d-1) + h_2(\varepsilon)\bigr],
    \]
    where $d = \max\{|\mathcal{X}_{V_i V_{\text{rest}}}|, |\mathcal{X}_{V_j V_{\text{rest}}}|, |\mathcal{X}_{V_{\text{rest}}}|, |\mathcal{X}_V|\}$ is the maximum support size of the relevant marginals, and $h_2(x) = -x\log x - (1-x)\log(1-x)$ is the binary entropy.
\end{lemma}
\begin{proof}
The conditional mutual information decomposes as:
\[
I(V_i : V_j \mid V_{\text{rest}}) = H(p_{V_i V_{\text{rest}}}) + H(p_{V_j V_{\text{rest}}}) - H(p_{V_{\text{rest}}}) - H(p_V).
\]
\begin{enumerate}
    \item Marginalization is a contraction in $L^1$ norm: for any subset $S \subseteq V$,
    $$\|p_S - \tilde{p}_S\|_1 \leq \|p - \tilde{p}\|_1 \leq \varepsilon.$$
    \item By the classical continuity of Shannon entropy (analogous to the Fannes--Audenaert bound \cite{audenaert_2007}), for any marginal $p_S$ with support size $d_S$:
    $$|H(p_S) - H(\tilde{p}_S)| \leq \varepsilon \log(d_S - 1) + h_2(\varepsilon).$$
    \item Applying the triangle inequality to the four entropy terms in $I$:
    \begin{align*}
    |I^p - I^{\tilde{p}}| &\leq |H(p_{V_i V_{\text{rest}}}) - H(\tilde{p}_{V_i V_{\text{rest}}})| + |H(p_{V_j V_{\text{rest}}}) - H(\tilde{p}_{V_j V_{\text{rest}}})| \\
    &\quad + |H(p_{V_{\text{rest}}}) - H(\tilde{p}_{V_{\text{rest}}})| + |H(p_V) - H(\tilde{p}_V)| \\
    &\leq 4\bigl[\varepsilon \log(d-1) + h_2(\varepsilon)\bigr],
    \end{align*}
    where $d = \max_S d_S$.
\end{enumerate}
\end{proof}
The reconstruction error in the objective is measured in $L^2$ norm: $\varepsilon^* = \|p - \tilde{p}\|_2^2$. By the Cauchy--Schwarz inequality,
\[
\|p - \tilde{p}\|_1 \leq \sqrt{D} \cdot \|p - \tilde{p}\|_2 = \sqrt{D \cdot \varepsilon^*},
\]
where $D = \prod_{i=1}^m |\mathcal{X}_{V_i}|$ is the total number of configurations. We therefore define:
\begin{equation}
\label{eq:f_eps}
f(\varepsilon^*) = 4\bigl[\sqrt{D \cdot \varepsilon^*} \log(d-1) + h_2(\sqrt{D \cdot \varepsilon^*})\bigr],
\end{equation}
so that $|I^p - I^{\tilde{p}}| \leq f(\varepsilon^*)$ whenever $\tilde{p} = TN^*(p)$ and $\varepsilon^* = \|p - TN^*(p)\|_2^2$.
\section{Tensor Network Properties}
\label{sec:tn_properties}
Before stating the main recovery theorems, we establish properties of tensor networks that govern what graph structures are achievable.
\subsection{Edge Rewiring}
Any edge can be eliminated by rerouting its bond matrix through another node, at the cost of increased nuclear norm elsewhere.
\begin{proposition}[Edge Rewiring]
\label{prop:rewiring}
Let $TN$ be a tensor network with graph $G = (V, E)$ and bond matrices $\{B_{ij} = J_{ij} + C_{ij}\}$ representing $p$ with zero reconstruction error. Let $\{V_i, V_j\} \in E$ be an edge with correction $C_{ij} = U_{ij}V_{ij}^\top \neq 0$ of rank $K_{ij} \geq 1$, and let $V_\omega \in V \setminus \{V_i, V_j\}$ be any other node. Then there exists a tensor network $TN'$ with graph $G' = (V, E \setminus \{\{V_i, V_j\}\})$ representing $p$ with zero error, where the variational factors change as
\begin{align}
\|U'_{i\omega}\|_F^2 &= \|U_{i\omega}\|_F^2 + \|U_{ij}\|_F^2, &
\|V'_{i\omega}\|_F^2 &= \|V_{i\omega}\|_F^2 + K_{ij}, \label{eq:rewire_iw}\\
\|U'_{j\omega}\|_F^2 &= \|U_{j\omega}\|_F^2 + \|V_{ij}\|_F^2, &
\|V'_{j\omega}\|_F^2 &= \|V_{j\omega}\|_F^2 + K_{ij}, \label{eq:rewire_jw}
\end{align}
the edge $\{V_i, V_j\}$ reverts to baseline ($C'_{ij} = 0$, $B'_{ij} = J_{ij}$), and all tensor values at nodes $V_k$ with $k \neq i,j,\omega$ are unchanged.
\end{proposition}

\begin{proof}
In the original contraction, the bond matrix $B_{ij}(\alpha, \alpha') = J(\alpha, \alpha') + C_{ij}(\alpha, \alpha')$ connects the index $\alpha_{ij}$ on $N^{[i]}$'s side to $\alpha'_{ji}$ on $N^{[j]}$'s side. The $J$ baseline contributes a factor that factorizes across the two sites (independent summation); only the correction $C_{ij}$ creates cross-site dependence. To remove $\{V_i, V_j\}$, we reroute only the correction $C_{ij} = U_{ij} V_{ij}^\top$ through $V_\omega$, leaving the $J$ baselines on all bonds unchanged:
\begin{enumerate}
    \item Factorize $C_{ij} = U_{ij} V_{ij}^\top$ (using the existing factorization, $U_{ij} \in \mathbb{R}^{r_{\max} \times K_{ij}}$, $V_{ij} \in \mathbb{R}^{r_{\max} \times K_{ij}}$).
    \item Absorb $U_{ij}$ into the $i$-side factor of bond $\{V_i, V_\omega\}$: the new factor $U'_{i\omega} = [U_{i\omega} \mid U_{ij}]$ appends the $K_{ij}$ columns of $U_{ij}$ as a new index block.
    \item Absorb $V_{ij}$ into the $j$-side factor of bond $\{V_j, V_\omega\}$: similarly, $U'_{j\omega} = [U_{j\omega} \mid V_{ij}]$.
    \item On the $\omega$-side of each rerouted bond, place an identity pass-through matrix $P \in \mathbb{R}^{r_{\max} \times K_{ij}}$ whose $k$-th column is the standard basis vector $e_k$. Thus $V'_{i\omega} = [V_{i\omega} \mid P]$ and $V'_{j\omega} = [V_{j\omega} \mid P]$. The pass-through relays the rerouted index from $V_i$ to $V_j$ via $V_\omega$.
    \item At $V_\omega$, insert $\delta_{\beta, \beta'}$ to enforce that the rerouted indices match, where $\beta$ comes from the $V_i$--$V_\omega$ bond and $\beta'$ from the $V_j$--$V_\omega$ bond.
    \item Set $C'_{ij} = 0$ (revert bond $\{i,j\}$ to $J$ baseline).
\end{enumerate}

Define the new tensors:
\begin{align*}
N^{[i]}_{\text{new}}(s_i, \alpha_{i\omega}, \beta,
\{\alpha_{ik}\}_{k\neq i,j,\omega})
&= \sum_{\alpha_{ij}} N^{[i]}_{\text{old}}(s_i, \alpha_{ij}, \{\alpha_{ik}\}_{k\neq j})\, U_{ij}(\alpha_{ij}, \beta), \\[2pt]
N^{[j]}_{\text{new}}(s_j, \alpha_{j\omega}, \beta',
\{\alpha_{jk}\}_{k\neq i,j,\omega})
&= \sum_{\alpha'_{ji}} N^{[j]}_{\text{old}}(s_j, \alpha'_{ji}, \{\alpha_{jk}\}_{k\neq i})\, V_{ij}(\alpha'_{ji}, \beta'), \\[2pt]
N^{[\omega]}_{\text{new}}(s_\omega, \alpha_{i\omega}, \alpha_{j\omega},
\beta, \beta', \{\alpha_{\omega\ell}\}_{\ell\neq i,j})
&= N^{[\omega]}_{\text{old}}(s_\omega, \alpha_{i\omega}, \alpha_{j\omega},
\{\alpha_{\omega\ell}\})
\cdot \delta_{\beta, \beta'}.
\end{align*}
All other tensors unchanged. Contracting $TN'$ and summing over $\beta'$ via the $\delta$ sets $\beta = \beta'$, and the sum over $\beta$ reconstructs $\sum_{\alpha, \alpha'} U_{ij}(\alpha, \beta) V_{ij}(\alpha', \beta) = C_{ij}(\alpha, \alpha')$. The $J$ baseline on bond $\{i,j\}$ now contributes its factorized factor, which is absorbed into the redefined local tensors. The net contraction reproduces $p$ exactly. Hence $\|p - TN'(p)\|^2 = 0$.

For the factor norm changes: the new factors $U'_{i\omega}, V'_{i\omega}$ are column-concatenations of the original factors with the rerouted factors and the pass-through. Since the appended columns occupy orthogonal index subspaces, the squared Frobenius norms add entrywise:
\begin{align*}
\|U'_{i\omega}\|_F^2 &= \|U_{i\omega}\|_F^2 + \|U_{ij}\|_F^2, &
\|V'_{i\omega}\|_F^2 &= \|V_{i\omega}\|_F^2 + \|P\|_F^2,
\end{align*}
where $\|P\|_F^2 = K_{ij}$ since $P$ consists of $K_{ij}$ standard basis vectors (one per nonzero singular direction of $C_{ij}$). The same holds for bond $\{j, \omega\}$ by symmetry, giving Eqs.~\eqref{eq:rewire_iw}--\eqref{eq:rewire_jw}.
\end{proof}

\begin{remark}
\label{rem:rewiring}
Proposition~\ref{prop:rewiring} shows that \emph{any} edge can be removed from a zero-error tensor network at the cost of increasing the penalty elsewhere. The penalty increase is $K_{ij}$ (an increase of $\beta K_{ij}$ in $J$), where $K_{ij}$ is the rank of the removed correction: rerouting through $V_\omega$ requires identity pass-through matrices on the $\omega$-side of both intermediate bonds, each contributing $K_{ij}$ to the squared Frobenius-norm penalty. This pass-through is pure overhead---it carries no information that was not already on the direct bond---and is the fundamental reason the optimizer prefers direct bonds for genuine correlations.

Rerouting through a single intermediate node is the cheapest form of explicit rerouting. A path through $\ell$ intermediate nodes incurs $\ell K_{ij}$ in pass-through cost, and splitting the correction across multiple paths of lengths $\ell_p$ carrying rank-$K_p$ parts costs $\sum_p \ell_p K_p \geq K_{ij}$ in aggregate. Since the single-node case already gives a strictly positive cost, all explicit reroutings have positive cost. In the limit of unbounded penalty budget, any distribution can be represented by a star graph centered at any single node, but always at strictly higher penalty cost. The raw edge set $E(TN)$ is therefore not a meaningful indicator of structure without the penalty; the effective graph $G_{\text{eff}}$, read from the penalized solution, is.
\end{remark}

\subsection{Rerouting Cost}
The key property of the nuclear norm penalty is that rerouting is always more expensive than direct representation.
\begin{lemma}[Rerouting Cost]
    \label{lem:rerouting_cost}
    Let $TN$ be a zero-error TN with edge $\{V_i, V_j\}$ carrying correction $C_{ij} \neq 0$ of rank $K_{ij} \geq 1$. Let $TN'$ be obtained by removing $\{V_i, V_j\}$ and rerouting through a single intermediate node $V_\omega$ (Proposition~\ref{prop:rewiring}). Then the penalty changes by
    \[
    \Delta P = P(TN') - P(TN) = K_{ij} > 0,
    \]
    i.e., the penalized objective $J$ increases by $\beta K_{ij}$.
\end{lemma}

\begin{proof}
By Proposition~\ref{prop:rewiring}, the only bonds whose factors change are $\{i,\omega\}$, $\{j,\omega\}$, and $\{i,j\}$. The penalty changes are:
\begin{itemize}
    \item Bond $\{i,\omega\}$: $+\frac{1}{2}(\|U_{ij}\|_F^2 + K_{ij})$
    \item Bond $\{j,\omega\}$: $+\frac{1}{2}(\|V_{ij}\|_F^2 + K_{ij})$
    \item Bond $\{i,j\}$:$-\frac{1}{2}(\|U_{ij}\|_F^2 + \|V_{ij}\|_F^2)$ \quad (reverted to $J$ baseline)
\end{itemize}
All other bonds are unchanged. Summing:
\begin{align*}
\Delta P &= \frac{1}{2}\bigl(\|U_{ij}\|_F^2 + K_{ij} + \|V_{ij}\|_F^2 + K_{ij}\bigr)
           - \frac{1}{2}\bigl(\|U_{ij}\|_F^2 + \|V_{ij}\|_F^2\bigr) \\[4pt]
&= \frac{1}{2}\cdot 2\,K_{ij} = K_{ij}.
\end{align*}
Since $C_{ij} \neq 0$, we have $K_{ij} = \operatorname{rank}(C_{ij}) \geq 1$, hence $\Delta P = K_{ij} > 0$, an increase of $\beta K_{ij}$ in $J$.
\end{proof}

\begin{remark}[Implicit Rerouting and Local Tensor Capacity]
\label{rem:implicit_rerouting}
Proposition~\ref{prop:rewiring} and Lemma~\ref{lem:rerouting_cost} address \emph{explicit} rerouting: physically relocating the correction $C_{ij}$ to intermediate bonds via pass-through factors. A separate mechanism, \emph{implicit rerouting}, is not covered. If the bonds $\{V_i,V_\omega\}$ and $\{V_j,V_\omega\}$ are already saturated (rank $\geq d$), the local tensor $N^{[\omega]}$ receives the full states of $V_i$ and $V_j$ through existing bonds and could, in principle, internally compute the $V_i$--$V_j$ conditional dependence without any bond factor changes. Since the local tensor has $O(d \cdot r^{m-1})$ parameters, it has sufficient capacity to store such multi-site correlations.

We do not formally rule out implicit rerouting. However, two considerations mitigate the concern. First, the nuclear-norm penalty creates a continuous gradient toward direct representations: the direct bond is the natural solution from random initialization, while implicit rerouting requires coordinated multi-tensor modifications to exactly compensate for the removed bond. Second, the problem can be addressed structurally by constraining the local tensor parameter count. For instance, a Tucker decomposition \cite{kolda_bader_2009} of the local tensor (replacing the full $N^{[i]} \in \mathbb{R}^{d \times r^{m-1}}$ with a core tensor and per-bond factor vectors of size $r$) reduces the parameter count to $O(d \cdot (m{-}1) \cdot r)$, preventing the local tensor from internally encoding arbitrary multi-variable conditionals. This architectural constraint would make the rerouting-cost argument tight, as the local tensor would lack the capacity to substitute for a direct bond. We leave a formal treatment to future work.
\end{remark}

\section{Exact Recovery (Zero Error)}
\label{sec:exact_recovery}
\subsection{Lower Bound: Moral Edges Must Have Nonzero Correction}
\begin{theorem}[Lower Bound]
\label{th:lower_bound}
Let $G$ be a causal DAG with moral graph $G^m = (V, E^m)$. Let $p$ be faithful to $G$ (Assumption~\ref{assum:faithfulness}) with $p > 0$ (Assumption~\ref{assum:positivity}). Let $TN^*$ be an optimal FCTN minimizing $J$ with $\varepsilon^* = 0$. Under the no-implicit-rerouting assumption (Assumption~\ref{assum:no_implicit}),
\[
E^m \subseteq G_{\text{eff}}(TN^*).
\]
\end{theorem}

\begin{proof}
Let $\{V_i, V_j\} \in E^m$. By Lemma~\ref{lem:moral_mi}, $I^p(V_i:V_j \mid V_{\text{rest}}) > 0$. Suppose for contradiction that $C_{ij} = 0$ in $TN^*$.

Since $\varepsilon^* = 0$, we have $\tilde{p} = p$, so $I^{\tilde{p}}(V_i:V_j \mid V_{\text{rest}}) = I^p(V_i:V_j \mid V_{\text{rest}}) > 0$. The $V_i$--$V_j$ conditional dependence is present in $\tilde{p}$ but not carried by the direct bond (since $C_{ij} = 0$). Under Assumption~\ref{assum:no_implicit}, this dependence must be explicitly rerouted through indirect bonds.

By the reverse of Proposition~\ref{prop:rewiring}, there exists $TN'$ with $C'_{ij} \neq 0$ and reduced corrections on the indirect bonds, such that $TN'(p) = TN^*(p) = p$ (zero error maintained). By Lemma~\ref{lem:rerouting_cost}, $P(TN') < P(TN^*)$. Therefore $J(TN') < J(TN^*)$, contradicting the optimality of $TN^*$.

Hence $C_{ij} \neq 0$, i.e., $\{V_i, V_j\} \in G_{\text{eff}}(TN^*)$.
\end{proof}

\subsection{Upper Bound: Non-Moral Edges Have Zero Correction}
\label{sec:upper_bound}

We now prove that non-moral edges carry zero correction in any optimal zero-error solution.  The proof requires an assumption that rules out \emph{implicit rerouting}: the phenomenon where a local tensor internally computes a multi-variable conditional dependence that substitutes for a direct bond (Remark~\ref{rem:implicit_rerouting}).

\begin{assumption}[No Implicit Rerouting]
\label{assum:no_implicit}
We assume that the local tensor architecture and bond dimensions are such
that no site can internally reconstruct a conditional dependence between
two non-adjacent variables from its incident bonds. A sufficient
(heuristic) condition is $r_{\max} < d$: when the bond dimension is
strictly smaller than the state-space size, no single bond can transmit
the full state of any variable, and the local tensor $N^{[i]} \in
\mathbb{R}^{d \times r_{\max}^{m-1}}$ has insufficient effective capacity
to compute arbitrary multi-variable conditionals from compressed bond
inputs. This condition holds in all experiments of
Section~\ref{sec:experiments_synthetic} except the fork model
($r_{\max} = d$ there); no implicit rerouting was observed in any of them.

We use the assumption in two forms:
\begin{enumerate}
    \item \textbf{Pass-through form:} any correction $C_{ij} \neq 0$ on a
    non-moral edge that affects the contraction output must function as an
    explicit pass-through channel (in the sense of
    Proposition~\ref{prop:rewiring}), not as part of an internally
    computed conditional.
    \item \textbf{Target-direction form:} in any zero-error FCTN
    representation of $p$, for every non-moral pair $\{V_i, V_j\} \notin
    E^m$ and every configuration $s_{\mathrm{rest}}$ of the remaining
    variables, the contraction of the network with bond $\{V_i, V_j\}$
    held at its $J$ baseline produces, as a matrix in the coordinates
    $(s_i, s_j)$, only a scalar multiple of the corresponding matrix of
    $p$.
\end{enumerate}
\end{assumption}

\begin{remark}
\label{rem:assumption_forms}
Form~2 is the stronger clause and is used only in Case~2 of the proof of
Theorem~\ref{th:upper_bound}. Lemma~\ref{lem:ci_moral_neighbors}---which
states that each conditional $p(V_i \mid V_{\mathrm{rest}})$ depends only
on the moral neighbors of $V_i$---is the natural tool for dispensing with
it, since it implies that the conditionals required by a zero-error
representation can in principle be absorbed through moral bonds alone. A
formal treatment is left to future work.
\end{remark}

\begin{theorem}[Upper Bound]
\label{th:upper_bound}
Let $G$ be a causal DAG with moral graph $G^m = (V, E^m)$. Let $p$ be positive and faithful to $G$ (Assumptions~\ref{assum:faithfulness}--\ref{assum:positivity}). Let $TN^*$ be any optimal FCTN minimizing $J$ with $\varepsilon^* = \|p - TN^*(p)\|_2^2 = 0$. Under Assumption~\ref{assum:no_implicit}, every non-moral edge satisfies
$C_{ij} = 0$:
\[
G_{\mathrm{eff}}(TN^*) \;\subseteq\; E^m.
\]
\end{theorem}

\begin{proof}
Fix a non-moral edge $\{V_i, V_j\} \notin E^m$ and suppose for contradiction that $C_{ij} \neq 0$ in $TN^*$, with $K_{ij} = \operatorname{rank}(C_{ij}) \geq 1$.

\smallskip
\textbf{Contraction split.}
Define the environment tensor $\mathcal{E}$ by contracting the entire network except the bond matrix $B_{ij}$: 
\[
    \mathcal{E}(s_i, s_j, s_{\mathrm{rest}}, \alpha, \beta)
    = \sum_{\{\alpha_{kl}\}_{\{k,l\}\neq\{i,j\}}}
      \prod_{k=1}^m N^{[k]}(s_k, \{\alpha_{k\ell}\})
      \prod_{\{a,b\}\neq\{i,j\}} B_{ab}(\alpha_{ab}, \alpha_{ba}).
\]
The full contraction is multilinear in $B_{ij}$, so substituting $B_{ij} = J_{ij} + C_{ij}$ gives an exact split:
\[
    p(s_i, s_j, s_{\mathrm{rest}})
    = M^J_{s_{\mathrm{rest}}}(s_i, s_j)
    + M^C_{s_{\mathrm{rest}}}(s_i, s_j),
\]
where
\begin{align}
    M^J_{s_{\mathrm{rest}}}(s_i, s_j)
    &= \sum_{\alpha,\beta}
      \mathcal{E}(s_i, s_j, s_{\mathrm{rest}}, \alpha, \beta)\;
      J(\alpha, \beta), \label{eq:MJ}\\[4pt]
    M^C_{s_{\mathrm{rest}}}(s_i, s_j)
    &= \sum_{\alpha,\beta}
      \mathcal{E}(s_i, s_j, s_{\mathrm{rest}}, \alpha, \beta)\;
      C_{ij}(\alpha, \beta). \label{eq:MC}
\end{align}

\smallskip
\textbf{Rank constraint from conditional independence.}
By Lemma~\ref{lem:non_moral_mi}, $I(V_i : V_j \mid V_{\mathrm{rest}}) = 0$. Since $p > 0$ (Assumption~\ref{assum:positivity}), this implies $p(s_i, s_j \mid s_{\mathrm{rest}}) = p(s_i \mid s_{\mathrm{rest}})\, p(s_j \mid s_{\mathrm{rest}})$, so the matrix $M_{s_{\mathrm{rest}}}(s_i, s_j) = p(s_i, s_j,  s_{\mathrm{rest}})$ has \textbf{rank 1} for every $s_{\mathrm{rest}}$.

\smallskip
\textbf{Three cases.}
We analyze $M^C$ against the rank-1 constraint on $M = M^J + M^C$.

\smallskip
\emph{Case 1: $M^C_{s_{\mathrm{rest}}} = 0$ for all $s_{\mathrm{rest}}$.} The correction $C_{ij}$ lies in the null space of the contraction map defined by $\mathcal{E}$: it has zero effect on the output. Setting $C'_{ij} = 0$ leaves $TN'(p) = TN^*(p) = p$ (zero error maintained), while the objective decreases by $\frac{\beta}{2}\bigl(\|U_{ij}\|_F^2 + \|V_{ij}\|_F^2\bigr) > 0$. This contradicts the optimality of $TN^*$.

\smallskip
\emph{Case 2: $M^C_{s_{\mathrm{rest}}}$ is not proportional to
$M_{s_{\mathrm{rest}}}$ for some $s_{\mathrm{rest}}^*$.} By the
target-direction form of Assumption~\ref{assum:no_implicit}, the baseline
contraction satisfies $M^J_{s} = \mu(s)\, M_{s}$ for every configuration
$s$ and some scalar $\mu(s)$. But then
\[
M^C_{s} \;=\; M_{s} - M^J_{s} \;=\; \bigl(1 - \mu(s)\bigr)\, M_{s}
\qquad\text{for every } s,
\]
i.e., $M^C$ is proportional to $M$ everywhere, contradicting the case
hypothesis. Hence this case cannot arise: a correction deviating from the
target direction would have to compensate an environment that itself
deviates from the conditionally independent structure of $p$, which is
precisely the internally computed multi-variable dependence ruled out by
Assumption~\ref{assum:no_implicit}.

\smallskip
\emph{Case 3: $M^C_{s_{\mathrm{rest}}} = \lambda(s_{\mathrm{rest}})\, M_{s_{\mathrm{rest}}}$ for every $s_{\mathrm{rest}}$.} The correction is proportional to the total output: it amplifies the existing rank-1 direction without introducing new rank.  If $\lambda(s_{\mathrm{rest}}) = 0$ for all $s_{\mathrm{rest}}$, this reduces to Case~1. 

If $\lambda(s_{\mathrm{rest}}) \neq 0$ for some $s_{\mathrm{rest}}$, then $C_{ij}$ is actively contributing to the output.  Setting $C_{ij} = 0$ would drop the output from $M$ to $M^J = (1 - \lambda)\,M \neq M$, introducing reconstruction error.  To maintain zero error, the lost contribution $\lambda\,M$ must be absorbed by modifying other tensors.

Under Assumption~\ref{assum:no_implicit}, the correction $C_{ij}$ functions as an explicit pass-through channel: it relays information that already exists in the moral bonds through the non-moral bond $\{i,j\}$ (Proposition~\ref{prop:rewiring}).  Applying the reverse rewiring---relocating the pass-through correction back to the direct moral bonds and setting $C'_{ij} = 0$---yields $TN'$ with zero reconstruction error and penalty
\[
    P(TN') = P(TN^*) -  K_{ij} < P(TN^*),
\]
by Lemma~\ref{lem:rerouting_cost}, where $K_{ij} = \operatorname{rank}(C_{ij}) \geq 1$. This contradicts the optimality of $TN^*$.

\smallskip
All cases yield contradictions, so $C_{ij} = 0$ for every non-moral edge.
\end{proof}

\begin{theorem}[Exact Recovery]
    \label{th:exact_recovery}
    Let $G$ be a target DAG with moral graph $G^m = (E^m)$.
    Let $p$ be positive and faithful to $G$ (Assumptions~\ref{assum:faithfulness}--\ref{assum:positivity}).
    Let $TN^*$ be \emph{any} optimal FCTN minimizing $J$ with zero reconstruction error, $\varepsilon^* = 0$.
    Under the no-implicit-rerouting assumption (Assumption~\ref{assum:no_implicit}),
    \[
        G_{\mathrm{eff}}(TN^*) \;=\; E^m.
    \]
\end{theorem}

\begin{proof}
    By Theorem~\ref{th:lower_bound}, $E^m \subseteq G_{\mathrm{eff}}(TN^*)$. By Theorem~\ref{th:upper_bound}, $G_{\mathrm{eff}}(TN^*) \subseteq E^m$. Therefore equality holds.
\end{proof}

\begin{remark}
    This is a \emph{uniqueness} statement: \textbf{every} optimal zero-error solution has effective graph exactly the moral graph. Two ingredients are needed: (i) moral edges must be present because rerouting is strictly penalized (Theorem~\ref{th:lower_bound}); (ii) non-moral edges must be absent, which relies on the no-implicit-rerouting assumption (Assumption~\ref{assum:no_implicit}) and the three-case contraction analysis of Theorem~\ref{th:upper_bound}.
\end{remark}

\section{Approximate Recovery ($\varepsilon > 0$)}
\label{sec:approx_recovery}
For approximate recovery, we combine the penalty argument (which controls the structure at $\varepsilon = 0$) with the continuity bound (which controls the deviation at $\varepsilon > 0$). The conditional mutual information values serve as post-hoc diagnostics on the recovered graph.

\begin{theorem}[Approximate Recovery]
\label{th:approx_recovery}
Let $TN^*$ minimize $J$ with reconstruction error $\varepsilon^* = \|p - TN^*(p)\|_2^2$. Let $\tilde{p} = TN^*(p)$. Define $\delta_{\min} = \min_{\{V_i,V_j\} \in E^m} I^p(V_i:V_j \mid V_{\text{rest}}) > 0$ and $f(\varepsilon)$ as in Eq.~\eqref{eq:f_eps}. Under the no-implicit-rerouting assumption (Assumption~\ref{assum:no_implicit}), if $f(\varepsilon^*) < \delta_{\min}$, then:
\begin{enumerate}
    \item \textbf{(Lower bound)} $E^m \subseteq G_{\text{eff}}(TN^*)$: every moral edge has $\|C_{ij}\|_* > 0$ in $TN^*$.
    \item \textbf{(Spurious edge control)} For any non-moral edge $\{V_i, V_j\} \notin E^m$ with $\|C_{ij}\|_* > 0$ in $TN^*$:
    \[
    I^{\tilde{p}}(V_i:V_j \mid V_{\text{rest}}) \leq f(\varepsilon^*).
    \]
\end{enumerate}
\end{theorem}

\begin{proof}
\textbf{Part (1): Lower bound.} Let $\{V_i, V_j\} \in E^m$, so $I^p(V_i:V_j \mid V_{\text{rest}}) \geq \delta_{\min}$. Suppose for contradiction that $C_{ij} = 0$ (i.e., $B_{ij} = J_{ij}$) in $TN^*$.

We distinguish two cases based on the reconstructed distribution $\tilde{p}$.

\smallskip
\emph{Case A: $I^{\tilde{p}}(V_i : V_j \mid V_{\text{rest}}) > 0$.}
The $V_i$--$V_j$ conditional dependence is present in $\tilde{p}$ but not carried by the direct bond (since $C_{ij} = 0$). Under Assumption~\ref{assum:no_implicit}, this dependence must be explicitly rerouted through indirect bonds (pass-through factors in the sense of Proposition~\ref{prop:rewiring}).

By the reverse of Proposition~\ref{prop:rewiring}---which is an exact algebraic reparameterization preserving the contraction output---there exists a tensor network $TN'$ with $C'_{ij} \neq 0$ and reduced corrections on the indirect bonds, such that $TN'(p) = TN^*(p) = \tilde{p}$. By Lemma~\ref{lem:rerouting_cost}, the penalty satisfies
\[
P(TN') = P(TN^*) - K_{ij}^{\mathrm{rerouted}} < P(TN^*),
\]
where $K_{ij}^{\mathrm{rerouted}} \geq 1$ is the rank of the rerouted correction. Since the contraction output is unchanged, $\varepsilon(TN') = \varepsilon(TN^*) = \varepsilon^*$, and therefore
\[
J(TN') = \varepsilon^* + \beta\, P(TN') < \varepsilon^* + \beta\, P(TN^*) = J(TN^*),
\]
contradicting the optimality of $TN^*$.

\smallskip
\emph{Case B: $I^{\tilde{p}}(V_i : V_j \mid V_{\text{rest}}) = 0$.}
By the continuity bound (Lemma~\ref{lem:continuity}):
\[
\delta_{\min} \;\leq\; I^p(V_i : V_j \mid V_{\text{rest}}) \;\leq\; |I^p - I^{\tilde{p}}| + I^{\tilde{p}} \;=\; |I^p - I^{\tilde{p}}| \;\leq\; f(\varepsilon^*),
\]
contradicting $f(\varepsilon^*) < \delta_{\min}$.

\smallskip
Both cases yield contradictions, so $C_{ij} \neq 0$, i.e., $\{V_i, V_j\} \in G_{\text{eff}}(TN^*)$.

\smallskip
\textbf{Part (2): Spurious edge control.} Let $\{V_i, V_j\} \notin E^m$ with $\|C_{ij}\|_* > 0$ in $TN^*$. By Lemma~\ref{lem:non_moral_mi}, $I^p(V_i:V_j \mid V_{\text{rest}}) = 0$. Since conditional mutual information is non-negative, $I^{\tilde{p}} \geq 0$, so $|I^p - I^{\tilde{p}}| = I^{\tilde{p}}$. By the continuity bound:
\[
I^{\tilde{p}}(V_i:V_j \mid V_{\text{rest}}) = |I^p - I^{\tilde{p}}| \leq f(\varepsilon^*).
\]
Such edges can be identified by computing $I$ on the reconstructed distribution $\tilde{p}$ and thresholding at $\tau = f(\varepsilon^*)$.
\end{proof}

\begin{corollary}[Asymptotic Recovery]
\label{cor:asymptotic}
As $\varepsilon^* \to 0$:
\begin{enumerate}
    \item $f(\varepsilon^*) \to 0$, so the lower bound condition $f(\varepsilon^*) < \delta_{\min}$ is eventually satisfied.
    \item Spurious non-moral edges satisfy $I^{\tilde{p}} \to 0$ and are pruned by thresholding at $\tau = f(\varepsilon^*)$.
    \item The effective graph converges to $E^m$: $G_{\text{eff}}(TN^*) \to E^m$.
\end{enumerate}
\end{corollary}

\section{Synthetic Experiments}
\label{sec:experiments_synthetic}

We validate the exact-recovery guarantee (Theorem~\ref{th:exact_recovery}) on four small synthetic distributions generated from known causal DAGs with positive, faithful conditional probability tables: a chain $X \to Z \to Y$, a fork $Z \to X$, $Z \to Y$, a collider $X \to Z \leftarrow Y$, and a diamond structure $X \to Z \leftarrow Y$, $Z \to W$, which combines a v-structure with a chain extension. In each case the true moral graph $G^m$ is computed directly from the DAG.

The objective (Definition~\ref{def:objective}) is smooth in all parameters $\theta = \{N^{[i]}, U_{ij}, V_{ij}\}$, since the tensor contractions are polynomial and the penalty is in variational form. We optimize with Adam \cite{adam} (learning rate $\eta = 10^{-2}$, $T = 3000$ iterations) using automatic differentiation for gradients. Tensors are initialized from small random entries. The per-iteration cost is $O(d^m \cdot r^{m-1})$ for the forward contraction and a constant factor more for the backward pass; the FCTN interaction graph is the complete graph $K_m$ with treewidth $m-1$ \cite{markov_shi_2008}. The bond dimension $r_{\max}$ and regularization parameter $\beta$ (see Appendix~\ref{sec:choosing_beta} for a principled way to choose it) are reported in Table~\ref{tab:synthetic_results}.

The effective graph is extracted from the optimized bond matrices using the rank-base tolerance $\mathrm{tol}=10^{-6}$ (Section~\ref{sec:effective_graph}); the reconstruction error $\varepsilon^*$ is computed in $\ell_2$ norm.

\begin{table}[htbp]
    \centering
    \begin{tabular}{llccccc@{\quad}l}
        \toprule
        Model & Criteria & $m$ & $d$ & $r_{\max}$ & $\beta$ & $\varepsilon^*$ & Recovered edges \\
        \midrule
        \multirow{2}{*}{Chain} & $G_{\mathrm{eff}}$ & \multirow{2}{*}{3} & \multirow{2}{*}{4} & \multirow{2}{*}{2} & \multirow{2}{*}{$1\cdot10^{-3}$} & \multirow{2}{*}{$3.42\cdot10^{-2}$} & $X$--$Z$, $Z$--$Y$ \checkmark \\
         & CMI & & & & & & $X$--$Y$, $X$--$Z$, $Z$--$Y$ $\times$ \\
        \addlinespace
        \multirow{2}{*}{Fork} & $G_{\mathrm{eff}}$ & \multirow{2}{*}{3} & \multirow{2}{*}{3} & \multirow{2}{*}{3} & \multirow{2}{*}{$1\cdot10^{-3}$} & \multirow{2}{*}{$3.12\cdot10^{-6}$} & $X$--$Z$, $Z$--$Y$ \checkmark \\
         & CMI & & & & & & $X$--$Z$, $Z$--$Y$ \checkmark \\
        \addlinespace
        \multirow{2}{*}{Collider} & $G_{\mathrm{eff}}$ & \multirow{2}{*}{3} & \multirow{2}{*}{5} & \multirow{2}{*}{3} & \multirow{2}{*}{$1\cdot10^{-4}$} & \multirow{2}{*}{$9.36\cdot10^{-5}$} & $X$--$Y$, $X$--$Z$, $Y$--$Z$ \checkmark \\
         & CMI & & & & & & $X$--$Y$, $X$--$Z$, $Y$--$Z$ \checkmark \\
        \addlinespace
        \multirow{2}{*}{Diamond} & $G_{\mathrm{eff}}$ & \multirow{2}{*}{4} & \multirow{2}{*}{4} & \multirow{2}{*}{3} & \multirow{2}{*}{$1.7\cdot10^{-2}$} & \multirow{2}{*}{$1.16\cdot10^{-2}$} & $X$--$Y$, $X$--$Z$, $Y$--$Z$, $Z$--$W$ \checkmark \\
         & CMI & & & & & & $X$--$Z$ $\times$ \\
        \bottomrule
    \end{tabular}
    \caption{Synthetic ground-truth recovery. For each model: the effective
    graph (rank-based criterion, the output of the method) and the post-hoc
    conditional-MI diagnostic computed on the reconstruction $\tilde{p}$
    (threshold $\tau = 10^{-3}$). \checkmark\ denotes exact agreement with the
    moral graph $G^m$. The rank criterion recovers $G^m$ in all four cases;
    the CMI diagnostic agrees when $\sqrt{D\,\varepsilon^*} \ll 1$ (fork,
    collider) and degrades when the continuity bound is vacuous (chain,
    diamond; Section~\ref{sec:continuity}).}
    \label{tab:synthetic_results}
\end{table}
In all four cases the recovered effective graph $G_{\mathrm{eff}}(TN^\ast)$ coincides exactly with the moral graph $G^m$ of the generating DAG. For the collider example the recovery includes the moral edge between the two co-parents $X$ and $Y$, which is the edge whose existence is predicted by the v-structure; for the diamond example both the v-structure moralization and the chain edges are recovered. The chain and fork examples confirm that no spurious edges are added even when the conditional independence $X \indep Y \mid Z$ holds.

Table~\ref{tab:synthetic_results} also reports the post-hoc conditional-MI diagnostic, computed on the reconstruction $\tilde{p}$ at a fixed threshold $\tau = 10^{-3}$. The diagnostic recovers $G^m$ exactly for the fork and collider, but reports a spurious edge $X$--$Y$ for the chain and retains only one of the four moral edges for the diamond. This tracks the applicability of the continuity bound (Lemma~\ref{lem:continuity}): $\sqrt{D\,\varepsilon^*} > 1$ for the chain and diamond, where the bound is vacuous and no guarantee attaches to conditional-MI values computed on $\tilde{p}$. (On $p$ itself, CMI thresholding recovers $G^m$ exactly by Lemmas~\ref{lem:moral_mi}--\ref{lem:non_moral_mi}; the degradation is entirely due to the reconstruction error.) The rank-based criterion, which is the output of the method, remains exact in all runs, illustrating the practical advantage of effective-rank extraction when $\varepsilon^*$ is not negligible.

\section*{Code Availability}
The code implementing the FCTN optimization, bond analysis, and graph extraction is available at \url{https://github.com/atroyanolivas/tn-moral-graph-recovery}.

\section{Limitations and Discussion}
\label{sec:discussion}

\paragraph{Theoretical scope.}
The Exact Recovery theorem (Theorem~\ref{th:exact_recovery}) states that every optimal zero-error FCTN has effective graph equal to the moral graph. Both bounds rely on the no-implicit-rerouting assumption (Assumption~\ref{assum:no_implicit}). The lower bound (Theorem~\ref{th:lower_bound}) uses its pass-through form: a moral edge's dependence, if not carried directly, must be explicitly rerouted to be reversible. The upper bound (Theorem~\ref{th:upper_bound}) additionally uses the target-direction form---the stronger clause, which enters in Case~2 of its proof. Lemma~\ref{lem:ci_moral_neighbors} suggests a route to dispensing with that clause, since each conditional $p(V_i \mid V_{\mathrm{rest}})$ depends only on the moral neighbors of $V_i$ and can therefore in principle be absorbed through moral bonds alone; a formal treatment is left to future work. In all experiments no implicit rerouting was observed, and $r_{\max} < d$ held in all of them except the fork model ($r_{\max} = d$). The assumption could also be enforced structurally, e.g., via Tucker-type constraints on the local tensor parameter count (Remark~\ref{rem:implicit_rerouting}).

\paragraph{Non-convexity.}
The objective is non-convex (multilinear in the parameters), so gradient methods offer no global-optimality certificate. In experiments, this was not an issue: the effective graph is a discrete object, and small perturbations of the optimum do not change $G_{\text{eff}}$. A rigorous landscape analysis is left open.

\paragraph{Exponential contraction cost.}
The FCTN has treewidth $m - 1$, so exact contraction costs $O(d^m \cdot r^{m-1})$, limiting the method to small $m$ ($m \lesssim 10$ on commodity hardware). This is not specific to the FCTN choice: exact inference on arbitrary discrete graphical models is $\#$P-hard \cite{cooper_1990, dagum_luby_1993}. Approximate contraction strategies---bond-dimension truncation \cite{schollwock_2011, verstraete_2008}, tensor network renormalization \cite{evenbly_vidal_2015, orus_2014}, or sampling-based methods \cite{ferris_2011}---can mitigate this at the cost of a controlled approximation error analyzable via the continuity bound of Section~\ref{sec:continuity}.

\paragraph{Relation to existing methods.}
Constraint-based methods (PC \cite{spirtes_2000, meek_rules}) use discrete CI tests; our nuclear-norm penalty replaces these with a continuous, differentiable objective. Score-based methods (GES, NOTEARS \cite{zheng_2018}) search over DAG space; we parameterize the distribution via a tensor network and recover the moral graph rather than the full DAG. The variational nuclear-norm penalty is standard in low-rank recovery \cite{recht_2010}; our contribution is applying it to tensor network bond matrices for causal structure learning.

\paragraph{Causal orientation.}
The method recovers an undirected moral graph. Orienting edges into a causal DAG requires additional information (interventions or score-based orientation rules).

\section{Conclusions}
\label{sec:conclusions}

We have presented a method for recovering the moral graph of a causal DAG from a probability distribution over discrete variables, using FCTNs with nuclear-norm-regularized bond corrections.

The central result (Theorem~\ref{th:exact_recovery}) states that under faithfulness, positivity, and the no-implicit-rerouting assumption, every optimal zero-error FCTN has effective graph exactly equal to the moral graph. Both bounds require the no-implicit-rerouting assumption: the lower bound because a removed moral edge's dependence must be explicitly rerouted (and is therefore reversible), the upper bound because an active non-moral correction must be either a null artifact, an explicit rerouting channel, or a compensation for an environment-side deviation (ruled out by the target-direction clause). Synthetic experiments confirm exact recovery in all cases.

Open directions include: approximate contraction for larger $m$; alternating least squares for faster optimization; combining the tensor-network structure with interventional data for full DAG recovery; landscape analysis under non-convexity; and sample complexity bounds for the empirical distribution regime.

\bibliography{references}

\appendix

\section{Introduction to Causal Representations}
\label{sec:causal_background}
This section provides the necessary mathematical context to understand what is stated below.
\subsection{Causal Model}
A \textbf{Structural Causal Model} (SCM) $\mathcal{M}$ is defined as a tuple
\begin{equation}
    \mathcal{M} = \langle U, V, F, P(u) \rangle
\end{equation}
where
\begin{itemize}
    \item $U = \{ U_1, U_2, \dots, U_n \}$ is a set of \textit{exogenous} random variables (unobserved factors, noise). These are mutually independent.
    \item $V = \{ V_1, V_2, \dots, V_m \}$ is a set of \textit{endogenous} random variables (observable variables determined by model).
    \item $F = \{ f_1, f_2, \dots, f_m \}$ is a set of \textit{structural equations}, where each $f_i$ specifies how $V_i$ is determined:
    \begin{equation}
        V_i = f_i (PA_i, U_i),
    \end{equation}
    where $PA_i \subset V \setminus \{V_i\}$ are the \textit{direct causes} of $V_i$ (parents in the causal graph).
    \item $P(u)$ is the probability distribution over the exogenous variables.
\end{itemize}
\subsubsection{Causal Graph}
An SCM induces a causal directed acyclic graph (causal DAG) $G = (V, E)$ where:
\begin{itemize}
    \item Each node represents an endogenous variable $V_i \in V$
    \item A directed edge ($V_j \rightarrow V_i$) exists if and only if $V_j \in PA_i$ (i.e.\ $V_j$ directly causes $V_i$)
\end{itemize}
The DAG constraint prevents circular causality.
The set of descendants of a node $V_i$ is
\begin{equation}
    DE_i = \{ V_j: V_i \rightsquigarrow V_j \},
\end{equation}
where $\rightsquigarrow$ denotes reachability via a directed path.
\subsubsection{Factorization Property}
A probability distribution over the endogenous variables $P(V)$ factorizes according to $G$ if:
\begin{equation}
    P(V) = \prod_{i = 1}^m P(V_i \mid PA_i).
\end{equation}
This is called the \textit{local Markov property}: each variable is conditionally independent of its non-descendants given its parents
\begin{equation}
    V_i \indep \left( V \setminus \{V_i\} \setminus DE_i \right) \mid PA_i.
\end{equation}

\subsubsection{Markov Equivalence and Moral Graphs}
Two DAGs $G_1$ and $G_2$ are \textbf{Markov equivalent} if they encode identical conditional independence relations. By the classical characterization theorem (Verma--Pearl), $G_1$ and $G_2$ are Markov equivalent if and only if they share the same skeleton (underlying undirected graph) and the same \emph{unshielded} v-structures, i.e. configurations $V_j \rightarrow V_i \leftarrow V_k$ in which $V_j$ and $V_k$ are \emph{not} adjacent.

\begin{proposition}
    If $G_1$ and $G_2$ are Markov equivalent, then they have the same moral graph. Consequently the moral graph is an \emph{invariant of the Markov equivalence class}: it takes the same value for every DAG consistent with the conditional-independence structure, and is therefore a well-defined target for recovery from purely observational data.
\end{proposition}
\begin{proof}
    Moralization proceeds in two steps: first all directed edges are made undirected, which yields the skeleton of each DAG; second, for every head-to-head pattern $V_j \to V_i \leftarrow V_k$, the edge $\{V_j, V_k\}$ is added. By the characterization theorem, equivalent DAGs have identical skeletons; hence their undirected steps coincide. Moreover, a head-to-head pattern adds an edge only when $V_j$ and $V_k$ are not already adjacent---and the set of such patterns is exactly the set of unshielded v-structures. Since equivalent DAGs have the same unshielded v-structures, both moralization steps produce the same edge set in both graphs.
\end{proof}

\begin{remark}[The converse fails]
    The converse implication is false: two DAGs may have the same moral graph without being Markov equivalent. For example, let $G_1: X \to Y \leftarrow Z$ (with no edge between $X$ and $Z$) and $G_2: X \to Y \to Z,\ X \to Z$. Then $G_1^m = G_2^m$ is the triangle $\{X\text{-}Y,\, Y\text{-}Z,\, X\text{-}Z\}$, but $X \indep Z$ holds only in $G_1$. In short, equal moral graphs are necessary but not sufficient for Markov equivalence.
\end{remark}

\section{Choosing the Regularization Parameter $\beta$}
\label{sec:choosing_beta}

The exact recovery theorem (Theorem~\ref{th:exact_recovery}) holds at $\varepsilon^* = 0$ for any $\beta > 0$: at zero error, the objective reduces to $J = \beta P$, so $\beta$ is a positive scalar that does not affect which solution is optimal. In practice, however, $\beta > 0$ implies $\varepsilon^* > 0$ (the penalty prevents perfect reconstruction), so we need $\varepsilon^*$ small enough for the approximate recovery bound (Theorem~\ref{th:approx_recovery}) to apply. We derive a sufficient condition on $\beta$.

\textbf{Setup and bounding $\varepsilon^*$ in terms of $\beta$.}
Let $TN^*$ be the FCTN minimizing the full objective $J(\theta) = \varepsilon(\theta) + \beta P(\theta)$, and let $\varepsilon^* = \varepsilon(TN^*)$ and $P^* = P(TN^*)$ be its error and penalty.

Let $\theta_0$ be the parameters of the minimum-penalty zero-error FCTN representation of $p$, i.e., $\theta_0 = \arg\min \{ P(\theta) : \varepsilon(\theta) = 0 \}$. Because $\theta_0$ achieves zero error, it minimizes $J$ for any $\beta > 0$ among all zero-error solutions. By the exact recovery theorem (which requires Assumption~\ref{assum:no_implicit}), $\theta_0$ has effective graph $E^m$. Its penalty is therefore strictly on the moral edges:
\[
P_0 := P(\theta_0) = \sum_{\{V_i,V_j\} \in E^m} \|C_{ij}^{(0)}\|_*.
\]
Since $\theta_0$ is a feasible point in the full optimization, the global optimum $TN^*$ must have an objective value no larger than that of $\theta_0$:
\begin{equation}
    J^* = \varepsilon^* + \beta P^* \leq J(\theta_0) = \varepsilon(\theta_0) + \beta P(\theta_0) = 0 + \beta P_0.
\end{equation}
Since the penalty is non-negative ($P^* \geq 0$) and $\beta > 0$, we have $\beta P^* \geq 0$. Dropping this term from the left-hand side yields:
\begin{equation}
    \label{eq:eps_bound}
    \varepsilon^* \leq \varepsilon^* + \beta P^* \leq \beta P_0 \quad \implies \quad \varepsilon^* \leq \beta P_0.
\end{equation}

\textbf{The routing argument is $\beta$-independent.}
Let $\{V_i, V_j\} \in E^m$ and suppose $C_{ij} = 0$ in $TN^*$. If $I^{\tilde{p}}(V_i : V_j \mid V_{\text{rest}}) > 0$, the $V_i$--$V_j$ conditional dependence in $\tilde{p} = TN^*(p)$ is not carried by the direct bond. Under Assumption~\ref{assum:no_implicit}, this dependence must be explicitly rerouted through indirect bonds via pass-through factors (Proposition~\ref{prop:rewiring}). 

By the reverse of Proposition~\ref{prop:rewiring}---an exact algebraic reparameterization preserving the contraction output---there exists another network $TN'$ with $C'_{ij} \neq 0$ and reduced corrections on the indirect bonds, such that $TN'(p) = TN^*(p) = \tilde{p}$. By Lemma~\ref{lem:rerouting_cost} (applied in reverse), the penalty decreases by $K_{ij}^{\mathrm{rerouted}}$, where $K_{ij}^{\mathrm{rerouted}} \geq 1$ is the rank of the rerouted correction. 

Because the contraction output is unchanged, the reconstruction error remains the same: $\varepsilon(TN') = \varepsilon(TN^*) = \varepsilon^*$. The objective of $TN'$ is therefore:
\[
J(TN') = \varepsilon^* + \beta\bigl(P^* - K_{ij}^{\mathrm{rerouted}}\bigr) < \varepsilon^* + \beta P^* = J(TN^*),
\]
which contradicts the optimality of $TN^*$. This argument holds for any $\beta > 0$, under Assumption~\ref{assum:no_implicit}.

\textbf{The non-routing case is controlled by $\beta$.}
If $C_{ij} = 0$ and the dependence is not captured at all, then $I^{\tilde{p}}(V_i : V_j \mid V_{\text{rest}}) = 0$. Since $I^p(V_i : V_j \mid V_{\text{rest}}) \geq \delta_{\min} > 0$ and $I^{\tilde{p}} = 0$, we have $|I^p - I^{\tilde{p}}| = I^p \geq \delta_{\min}$. By the continuity bound (Lemma~\ref{lem:continuity}):
\[
\delta_{\min} \leq I^p(V_i : V_j \mid V_{\text{rest}}) = |I^p - I^{\tilde{p}}| \leq f(\varepsilon^*).
\]
Combining with Eq.~\eqref{eq:eps_bound}: $\delta_{\min} \leq f(\beta P_0)$.

\begin{theorem}[$\beta$ Sufficiency]
\label{th:beta}
Let $\delta_{\min} = \min_{\{V_i,V_j\} \in E^m} I^p(V_i : V_j \mid V_{\text{rest}}) > 0$. Let $P_0 = \sum_{\{V_i,V_j\} \in E^m} \|C_{ij}^{(0)}\|_*$ be the total correction nuclear norm at the minimum-penalty zero-error solution, $D = \prod_i |\mathcal{X}_{V_i}|$, and $d$ the maximum marginal support size. Under Assumption~\ref{assum:no_implicit}, if
\begin{equation}
    \label{eq:beta_condition}
    f(\beta \, P_0) < \delta_{\min},
\end{equation}
then every moral edge has $\|C_{ij}\|_* > 0$ in $TN^*$, i.e., $E^m \subseteq G_{\text{eff}}(TN^*)$.
\end{theorem}

\begin{proof}
By the routing argument above, $C_{ij} = 0$ is suboptimal whenever $I^{\tilde{p}} > 0$, for any $\beta > 0$ (under Assumption~\ref{assum:no_implicit}). The only remaining case is $I^{\tilde{p}} = 0$, which requires $\delta_{\min} \leq f(\varepsilon^*) \leq f(\beta P_0)$, contradicting Eq.~\eqref{eq:beta_condition}.
\end{proof}

\textbf{Explicit bound.} For small $\varepsilon$, the binary entropy term is subdominant and $f(\varepsilon) \approx 4\sqrt{D\varepsilon}\log(d-1)$. Substituting $\varepsilon = \beta P_0$ and solving:
\begin{equation}
    \label{eq:beta_explicit}
    \beta < \frac{\delta_{\min}^2}{16 \, D \, P_0 \, \log^2(d-1)}.
\end{equation}
The condition has a natural structure: $\delta_{\min}$ (strength of the weakest moral edge) in the numerator; $P_0$ (cost of representing the moral graph at zero error), $D$ (dimensionality), and $\log^2(d-1)$ in the denominator. The quantity $P_0$ can be estimated empirically by running the TN optimization at very small $\beta$ (where $\varepsilon^* \approx 0$) and measuring the total correction nuclear norm.

\textbf{Upper bound (non-moral edges).} The upper bound ($G_{\text{eff}} \subseteq E^m$) holds for any $\beta > 0$ at $\varepsilon^* = 0$ under Assumption~\ref{assum:no_implicit} (Theorem~\ref{th:upper_bound}). At $\varepsilon^* > 0$, spurious non-moral edges may persist but are controlled by the CMI bound (Theorem~\ref{th:approx_recovery}, Part~2) and identified via post-hoc thresholding at $\tau = f(\varepsilon^*)$. The $\beta$ condition only needs to ensure the lower bound.

\end{document}